\pdfoutput=1

\documentclass[10pt,twocolumn,letterpaper]{article}
\usepackage{cvpr}              

\usepackage{graphicx}
\usepackage{amsmath}
\usepackage{amssymb}
\usepackage{amsthm}

\usepackage{booktabs}
\usepackage{capt-of}
\usepackage{multirow}     
\usepackage{booktabs}     
\usepackage{colortbl}     
\usepackage{xcolor}       
\usepackage[pagebackref,breaklinks,colorlinks]{hyperref}
\usepackage{algorithm}
\usepackage{algorithmic}

\newtheorem{theorem}{Theorem}

\newtheorem{proposition}{Proposition}
\newtheorem{corollary}{Corollary}

\usepackage[capitalize]{cleveref}
\crefname{section}{Sec.}{Secs.}
\Crefname{section}{Section}{Sections}
\Crefname{table}{Table}{Tables}
\crefname{table}{Tab.}{Tabs.}
\usepackage{graphicx}
\usepackage{amsmath}
\usepackage{amssymb}
\usepackage{booktabs}
\usepackage{capt-of}
\usepackage{multirow}     
\usepackage{booktabs}     
\usepackage{colortbl}     
\usepackage{xcolor}       
\usepackage[pagebackref,breaklinks,colorlinks]{hyperref}

\usepackage[capitalize]{cleveref}
\crefname{section}{Sec.}{Secs.}
\Crefname{section}{Section}{Sections}
\Crefname{table}{Table}{Tables}
\crefname{table}{Tab.}{Tabs.}

\definecolor{cvprblue}{rgb}{0.21,0.49,0.74}

\def\paperID{14643} 
\def\confName{CVPR}
\def\confYear{2026}

\title{Upsample Anything: A Simple and Hard to Beat Baseline for Feature Upsampling}

\author{
Minseok Seo$^{1}$ \quad
Mark Hamilton$^{2,3}$ \quad
Changick Kim$^{1}$ \\
[0.25em]
$^{1}$KAIST \qquad
$^{2}$MIT \qquad
$^{3}$Microsoft \\
[0.5em]
{\tt\small minseok.seo@kaist.ac.kr}
}

\begin{document}
\twocolumn[{%
\renewcommand\twocolumn[1][]{#1}  %
\maketitle
\vspace{-2em}
\begin{center}
    \includegraphics[width=1.0\textwidth]{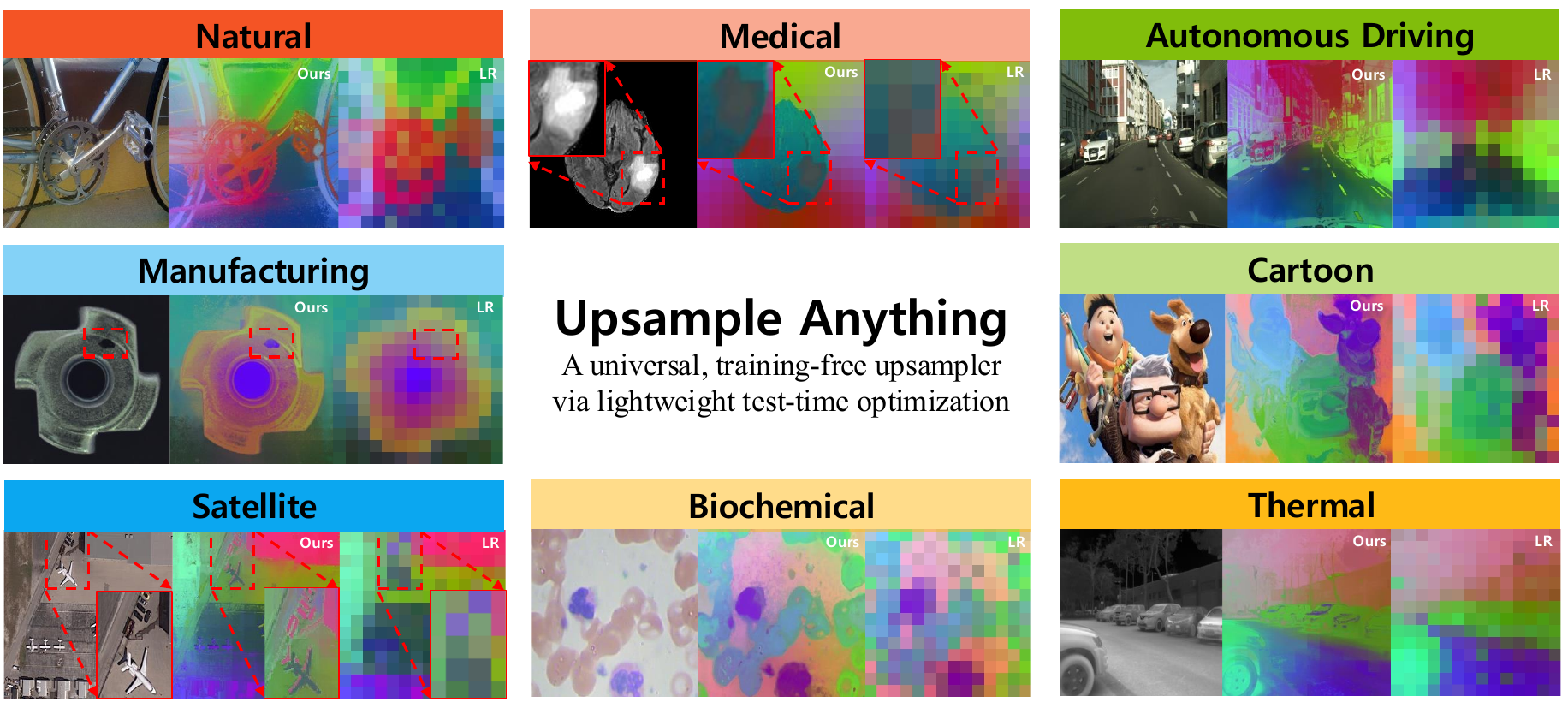}
    \captionof{figure}{Our method performs lightweight test-time optimization \textbf{($\approx$0.419 s/image)} without requiring any dataset-level training.It generalizes seamlessly across domains while maintaining consistent reconstruction quality for every image. (All examples are randomly selected, without cherry-picking.)
    }
    \label{fig:teaser}
    \vspace{1em}
\end{center}
}]
\begin{abstract}
We present \textbf{Upsample Anything}, a lightweight test-time optimization (TTO) framework that restores low-resolution features to high-resolution, pixel-wise outputs without any training.
Although Vision Foundation Models demonstrate strong generalization across diverse downstream tasks, their representations are typically downsampled by 14×/16× (e.g., ViT), which limits their direct use in pixel-level applications.
Existing feature upsampling approaches depend on dataset-specific retraining or heavy implicit optimization, restricting scalability and generalization.
Upsample Anything addresses these issues through a simple per-image optimization that learns an anisotropic Gaussian kernel combining spatial and range cues, effectively bridging Gaussian Splatting and Joint Bilateral Upsampling.
The learned kernel acts as a universal, edge-aware operator that transfers seamlessly across architectures and modalities, enabling precise high-resolution reconstruction of features, depth, or probability maps.
It runs in only $\approx0.419 \text{s}$ per 224×224 image and achieves state-of-the-art performance on semantic segmentation, depth estimation, and both depth and probability map upsampling.
\textbf{Project page:} \href{https://seominseok0429.github.io/Upsample-Anything/}{https://seominseok0429.github.io/Upsample-Anything/}

\end{abstract}    
\section{Introduction}
\label{sec:intro}
\begin{figure*}[t!]
    \centering
\includegraphics[width=2.0\columnwidth]{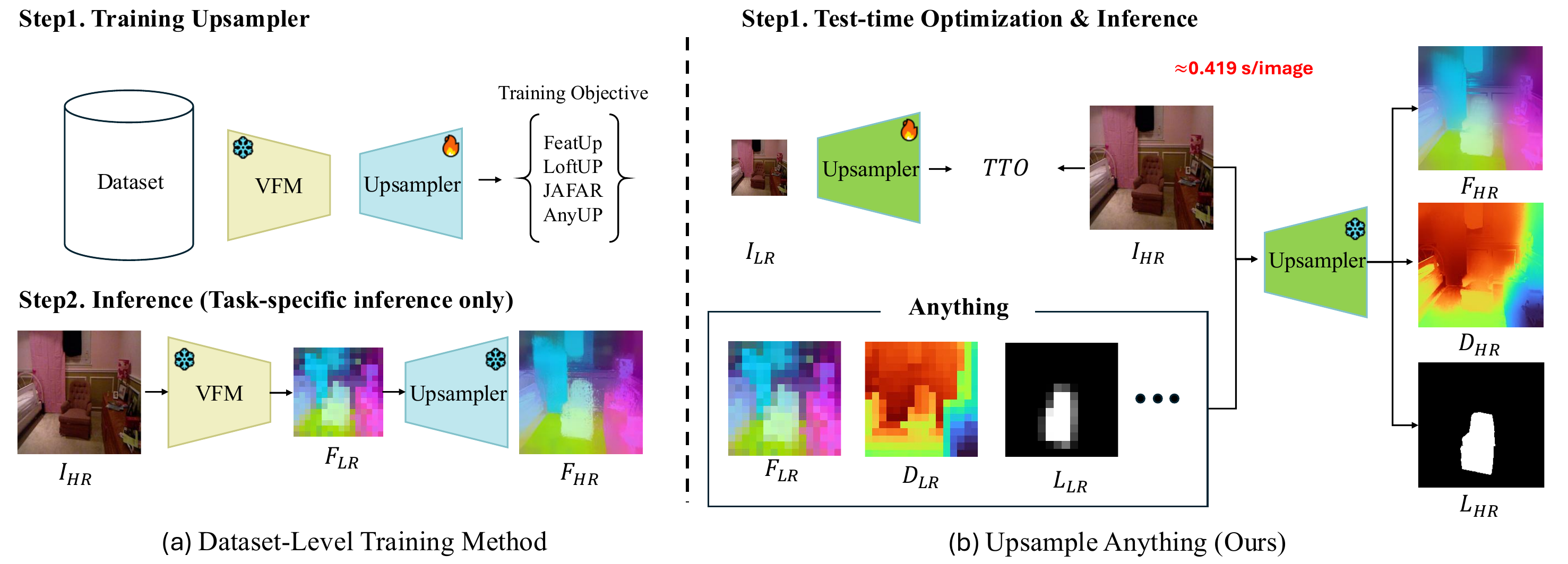}
    \caption{Comparison of dataset-level training and our test-time optimization (TTO). (a) Dataset-level methods (FeatUp, LoftUp, JAFAR, AnyUp) require paired training data and handle only 2D feature maps. (b) Our Upsample Anything performs TTO using only one HR image and generalizes to feature, depth, segmentation, and even 3D features.}
    \label{fig:fig2}
\end{figure*}
Modern computer vision systems for pixel-level prediction tasks such as semantic, instance, and panoptic segmentation~\cite{everingham2010pascal,lin2014microsoft,cordts2016cityscapes,zhou2019semantic} or depth estimation~\cite{silberman2012indoor,yang2024depth,ke2024repurposing} often use an encoder–decoder paradigm.
The encoder extracts hierarchical features that capture semantic abstraction from the input image, and the decoder reconstructs dense, task-specific predictions such as class maps, depth, or optical flow at the original spatial resolution.

Recent advances in large-scale self-supervised learning~\cite{caron2021emerging,oquab2023dinov2,he2022masked,zhou2021ibot,radford2021learning,zhai2023sigmoid} introduce general-purpose encoders called Vision Foundation Models (VFMs) that can serve as universal backbones across diverse downstream tasks.
This paradigm shift has led to the emergence of Vision Foundation Models (VFMs) such as DINO~\cite{oquab2023dinov2}, CLIP~\cite{radford2021learning}, SigLIP~\cite{zhai2023sigmoid}, and MAE~\cite{he2022masked}, which provide transferable and semantically rich features with minimal task-specific fine-tuning.

By decoupling the encoder from the downstream task, these foundation models dramatically reduce the data and training costs needed for adaptation while maintaining strong generalization across domains.
However, despite these advantages, high-performing pixel-level systems still require large and complex decoders such as DPT~\cite{ranftl2021vision}, UPerNet~\cite{xiao2018unified}, or SegFormer~\cite{xie2021segformer} to recover spatial details from low-resolution features.
Foundation features are typically downsampled by a factor of 14-16 in Vision Transformer~\cite{dosovitskiy2020image} architectures or equivalently through multiple pooling stages in CNN-based backbones~\cite{liu2022convnet,he2016deep}.
As a result, they lack fine-grained spatial information, forcing decoders to rely on heavy upsampling networks that are computationally expensive, memory-intensive, and often difficult to generalize to new architectures or resolutions.

To address this resolution gap, a growing line of research has explored feature upsampling methods~\cite{fu2024featup, suri2024lift,huang2025loftuplearningcoordinatebasedfeature,couairon2025jafar,wimmer2025anyup} to restore spatial details in pretrained representations without modifying the encoder.
These methods learn an upsampling operator that maps low-resolution foundation features to higher resolutions, effectively bridging the semantic–spatial gap before the downstream decoder.
In doing so, they can achieve strong performance across diverse pixel-level tasks even with a single 1×1 convolutional decoder.

Feature upsampling approaches can be broadly categorized into two paradigms depending on how the upsampler is optimized: (a) dataset-level training~\cite{fu2024featup, suri2024lift,huang2025loftuplearningcoordinatebasedfeature,couairon2025jafar,wimmer2025anyup} and (b) test-time optimization (TTO)~\cite{fu2024featup}, as illustrated in ~\cref{fig:fig2}.
In the dataset-level training paradigm, the feature upsampler is trained on a target dataset either by generating pseudo-labels using methods such as SAM~\cite{kirillov2023segment} for zero-shot supervision~\cite{huang2025loftuplearningcoordinatebasedfeature} or by adopting multi-view training objectives~\cite{fu2024featup, suri2024lift, couairon2025jafar, wimmer2025anyup}.

While this approach can generalize to certain unseen data, it still requires dataset-level training, meaning that the upsampler must be retrained whenever the backbone architecture or target dataset changes.
Moreover, due to heavy memory usage, most trained upsamplers can only operate up to 112–224 pixels in resolution.
The test-time optimization paradigm, exemplified by methods such as FeatUp (Implicit), avoids dataset-level training by optimizing the feature upsampler directly at inference time for each test image.
Although this removes the need for offline training, the per-image optimization is computationally expensive, taking an average of 49 seconds to converge for a 224-sized image.

We propose Upsample Anything, a test-time optimization (TTO) framework for feature upsampling, as illustrated in~\cref{fig:fig2}-(b).
Unlike previous methods requiring dataset-level training, it performs lightweight per-image optimization and processes a 224-sized image in only $\approx$ 0.419 s.
Given an input image, Upsample Anything resizes the RGB guidance to match the low-resolution (LR) feature-map size, reconstructs the high-resolution (HR) color image through optimization, and learns pixelwise anisotropic Gaussian parameters—$(\sigma_x, \sigma_y, \theta, \sigma_r)$—that define a continuous spatial–range splatting kernel.
These optimized kernels are then applied to the LR feature maps from a foundation encoder to produce HR feature maps aligned with the original image grid.
Although the optimization is guided only by color reconstruction, the learned kernels implicitly capture geometry and semantics.
As a result, Upsample Anything not only enhances 2D feature resolution but also generalizes to other pixel- or voxel-level signals (e.g., depth, segmentation, or even 3D representations) without retraining.
This property highlights its potential as a unified, lightweight, and resolution-free upsampling operator across 2D and 3D domains.
Despite requiring no dataset-level training, it consistently achieves state-of-the-art or near-SOTA performance on multiple pixel-level benchmarks, including semantic segmentation and depth estimation.
\section{Related Works}

\subsection{Joint Bilateral Upsampling}

Joint Bilateral Upsampling (JBU), first introduced by~\cite{kopf2007joint}, is a classic non-learning, edge-preserving upsampling technique designed to transfer structural details from a high-resolution guidance image to a low-resolution signal such as a depth map or label map.
Formally, JBU computes each high-resolution output pixel $\hat{F}{hr}[p]$ as a weighted average of nearby low-resolution pixels $F{lr}[q]$, as follows:
{\footnotesize
\begin{equation}
\hat{F}_{hr}[p] = \frac{1}{Z_p} \sum_{q \in \Omega(p)} F_{lr}[q]
\exp\!\left(-\frac{\|p - q\|^2}{2\sigma_s^2}\right)
\exp\!\left(-\frac{\|I[p] - I[q]\|^2}{2\sigma_r^2}\right),
\label{eq:jbu}
\end{equation}
}
where $I$ denotes the high-resolution guidance image, and $\sigma_s$, $\sigma_r$ control the spatial and range sensitivity, respectively.
Here, $Z_p$ is a normalization factor ensuring that all weights sum to one, and $\Omega(p)$ denotes the spatial neighborhood around pixel $p$ in the low-resolution domain.
By coupling spatial proximity and color similarity, JBU preserves edges and fine details while interpolating missing information, enabling high-quality restoration of dense signals without any additional learning.
This formulation has inspired numerous modern variants and learnable extensions that generalize bilateral filtering~\cite{fu2024featup} to feature space and neural representations.
Although JBU performs upsampling without any training, the resulting quality remains limited due to its fixed, hand-crafted kernel design.

\subsection{Feature Upsampling}
A pioneering work in feature upsampling is \textbf{FeatUp}~\cite{fu2024featup}, which proposed two model-agnostic modules: \textit{FeatUp (JBU)} and \textit{FeatUp (Implicit)}. 
The former generalizes Joint Bilateral Upsampling (JBU)~\cite{kopf2007joint} to high-dimensional feature space by replacing the fixed Gaussian range kernel with a learnable MLP, while the latter parameterizes high-resolution features as an implicit function $F_{hr}=\text{MLP}(x,I(x))$ optimized per image. 
While FeatUp effectively restores spatial detail from foundation features, it either requires dataset-level training or incurs long per-image optimization.

Several follow-up studies further improved spatial fidelity through learnable upsamplers. 
\textbf{LiFT}~\cite{suri2024lift} employed a lightweight U-Net-like upsampler with reconstruction loss, 
\textbf{LoftUp}~\cite{huang2025loftuplearningcoordinatebasedfeature} integrated RGB coordinates via cross-attention with pseudo-groundtruth supervision, 
and \textbf{JAFAR}~\cite{couairon2025jafar} introduced joint-attention filtering for semantic–structural alignment. 
\textbf{AnyUp}~\cite{wimmer2025anyup} proposed resolution-conditioned kernels for scalable upsampling. 
Despite their strong performance, these methods rely on dataset-level training, which makes them less adaptive to unseen domains.
They often generalize reasonably well but still exhibit suboptimal performance when facing novel architectures, resolutions, or out-of-distribution data.

\begin{figure*}[t!]
    \centering
    \includegraphics[width=1.65\columnwidth]{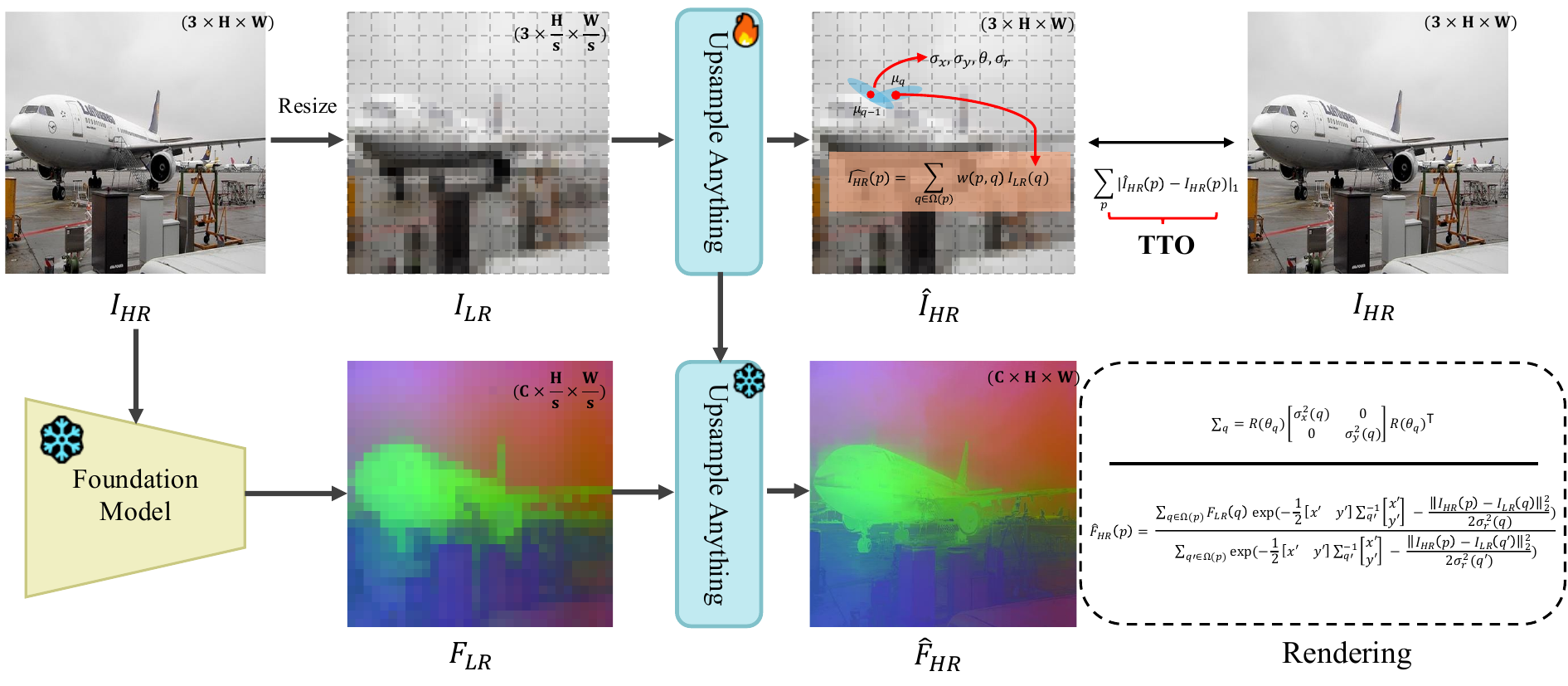}
    \caption{
\textbf{Overview of Upsample Anything.}
Given a high-resolution image $I_{hr}$, we downsample it to $I_{lr}$ and optimize GSJBU to reconstruct $I_{hr}$, learning per-pixel anisotropic kernels 
$\{\sigma_x, \sigma_y, \theta, \sigma_r\}$ via test-time optimization (TTO).
The learned kernels are then applied to foundation features $F_{lr}$ for rendering the high-resolution features $F_{hr}$, achieving pixel-wise anisotropic joint bilateral upsampling.
}
\label{fig:main}
\end{figure*}

\section{Preliminaries}
\paragraph{2D Gaussian Splatting (2DGS).}
Recent works extend 3D Gaussian Splatting (3DGS)~\cite{kerbl20233d} from volumetric radiance fields to 2D image representations~\cite{zhang2024gaussianimage,zhang2025image,zhu2025large}.
In the image–plane setting, a pixel or small region is represented by a Gaussian kernel
\begin{equation}
G_i(x) = \exp\!\Big(-\tfrac{1}{2}(x-\mu_i)^{\top}\Sigma_i^{-1}(x-\mu_i)\Big),
\end{equation}
where $\mu_i \in \mathbb{R}^2$ is the center, $\Sigma_i \in \mathbb{R}^{2\times2}$ is a positive–definite covariance (encoding scale and orientation), and $\alpha_i$ is a per–kernel weight.
The rendered image (or feature map) is obtained by normalized alpha blending:
\begin{equation}
I(x) = \sum_i w_i,c_i, \qquad
w_i = \frac{\alpha_i G_i(x)}{\sum_j \alpha_j G_j(x)},
\end{equation}
with $c_i$ the color or feature associated with kernel $i$.
Because all kernels lie on a single 2D plane, rendering reduces to a normalized weighted summation without depth sorting. This process is often described as rasterization/alpha blending and can also be interpreted as a spatially varying anisotropic convolution.
This property enables real-time, pose-free optimization directly in 2D image/feature domains.

\paragraph{Relation to Joint Bilateral Upsampling (JBU).}
JBU~\cite{kopf2007joint} computes each high-resolution (HR) output $F_{\mathrm{hr}}(p)$ as a normalized weighted average of low-resolution (LR) samples $F_{\mathrm{lr}}(q)$, where the weights depend on both spatial proximity and guidance-image similarity (see Eq.~\ref{eq:jbu}).
Viewed geometrically, each LR sample $q$ can be regarded as a Gaussian centered at $\mu_q{=}q$ with an \emph{isotropic} covariance $\Sigma_q{=}\sigma_s^2 I$.
Under this view, JBU corresponds to a discrete and isotropic instance of a guidance-modulated 2D Gaussian Splatting (2DGS) process, where the range term is provided by the HR guidance image.
In contrast, Upsample Anything assigns per-pixel anisotropic covariances $\Sigma_p$ and range scales $\sigma_{r,p}$ through test-time optimization, enabling adaptive fusion across space and range.
This pixel-level Gaussian parameterization captures locally varying orientations and scales, making Upsample Anything a continuous, edge-preserving extension of JBU within the 2DGS framework.
For the exact derivation and a more formal discussion of how this formulation differs from a simple combination, please refer to Appendix~\textsection12.

\section{Methods}

\subsection{Overview}
As illustrated in \cref{fig:main}, our method, Upsample Anything, consists of two stages: 
(i) \textbf{test-time optimization (TTO)} and 
(ii) \textbf{feature rendering}.
In the TTO stage, Upsample Anything learns per-pixel anisotropic Gaussian parameters 
$\{\sigma_x, \sigma_y, \theta, \sigma_r\}$ 
by reconstructing the high-resolution image $I_{hr}$ from its patch-wise downsampled version $I_{lr}$.
This process enables each pixel to learn how spatially and photometrically similar neighbors should be blended-effectively discovering local mixing weights that generalize beyond the image domain.
Once optimized, these Gaussian kernels are directly transferred to the foundation feature space, 
where the low-resolution feature map $F_{lr}\!\in\!\mathbb{R}^{C\times H/s\times W/s}$ 
is splatted to produce the high-resolution feature $F_{hr}\!\in\!\mathbb{R}^{C\times H\times W}$ 
using the same learned anisotropic weighting mechanism.
Because the splatting weights depend only on spatial–range similarity, 
this transfer is naturally domain-agnostic, allowing the learned kernels to act as universal upsampling operators.
Excluding the feature extraction time of the Vision Foundation Model (VFM), 
the entire optimization and inference for a $224{\times}224$ image takes  $\approx 0.419$\,s.
\subsection{Algorithm Design}
\label{sec:al}
Our design is inspired by classical Joint Bilateral Upsampling (JBU)~\cite{kopf2007joint}.
The key merit of JBU is \emph{transferability}: it does not hallucinate new values but instead \emph{learns mixing weights} that decide how much to blend neighboring samples, which makes it naturally model-/task-agnostic.
However, standard JBU in ~\cref{eq:jbu} is limited by \emph{global} $(\sigma_s,\sigma_r)$ and \emph{isotropic} range/spatial kernels, reducing expressivity near complex structures.

\paragraph{Per-pixel anisotropic kernels.}
To overcome these limits, Upsample Anything assigns a \emph{per-LR-pixel} anisotropic Gaussian with parameters
$\{\sigma_x(q),\,\sigma_y(q),\,\theta(q),\,\sigma_r(q)\}$ for each low-resolution location $q$.
Let the spatial covariance be
{\footnotesize
\begin{equation}
\Sigma_q =
R(\theta_q)
\begin{bmatrix}
\sigma_x^2(q) & 0 \\[3pt]
0 & \sigma_y^2(q)
\end{bmatrix}
R^\top(\theta_q);
\quad
R(\theta_q) =
\begin{bmatrix}
\cos\theta_q & -\sin\theta_q \\[3pt]
\sin\theta_q & \cos\theta_q
\end{bmatrix}.
\end{equation}
}

For an HR coordinate $p$, the \emph{unnormalized} spatial weight and range (guidance) weight are
\begin{align}
\small
\log w^{\text{s}}_{p\leftarrow q}
&= -\tfrac{1}{2}\,(p-\mu_q)^\top \Sigma_q^{-1}(p-\mu_q), \\
\log w^{\text{r}}_{p\leftarrow q}
&= -\frac{\|I(p)-I(q)\|^2}{2\,\sigma_r^2(q)}.
\end{align}

where $\mu_q$ is the LR center projected to HR coordinates and $I(\cdot)$ is the HR guidance image.
The final normalized mixing weight is
\begin{equation}
w_{p\leftarrow q} \;=\; \frac{\exp\!\big(\log w^{\text{s}}_{p\leftarrow q}+\log w^{\text{r}}_{p\leftarrow q}\big)}%
{\sum\limits_{q'\in\Omega(p)} \exp\!\big(\log w^{\text{s}}_{p\leftarrow q'}+\log w^{\text{r}}_{p\leftarrow q'}\big)}.
\end{equation}

\paragraph{Feature rendering (no value synthesis).}
Given low-resolution features $F_{lr}\in\mathbb{R}^{C\times H_l\times W_l}$ and scale $s$, we render the HR feature $F_{hr}\in\mathbb{R}^{C\times (sH_l)\times (sW_l)}$ by \emph{pure mixing}:
\[
F_{hr}(p) \;=\; \sum_{q\in\Omega(p)} w_{p\leftarrow q}\; F_{lr}(q).
\]
This strictly reweights existing LR features (no content generation), hence transfers across backbones and tasks.

\paragraph{Why it generalizes.}
Unlike feed-forward upsamplers that require dataset-level training~\cite{fu2024featup, suri2024lift, huang2025loftuplearningcoordinatebasedfeature, couairon2025jafar, wimmer2025anyup},
Upsample Anything learns only per-image, pixel-wise mixing weights from the HR guidance through a test-time optimization process, and reuses these weights to splat $F_{lr}$ into $F_{hr}$.
Because the mechanism is based on edge- and range-aware interpolation rather than value synthesis,
Upsample Anything is inherently \emph{resolution-free, model-agnostic}, and robust to unseen domains.

\subsection{Test-Time Optimization}
After defining the Upsample Anything formulation in~\cref{sec:al}, the next step is to optimize its per-pixel parameters
$\{\sigma_x, \sigma_y, \theta, \sigma_r\}$.
Our key idea is inspired by the patchified processing of modern Vision Foundation Models (VFMs): 
since VFMs downsample images by a fixed stride to extract low-resolution features, 
we emulate this process during optimization.

Specifically, the high-resolution image $I_{hr}$ is downsampled to $I_{lr}$ by bilinear interpolation with a stride $s$,
and the GSJBU parameters are optimized under a reconstruction objective from $I_{lr}$ back to $I_{hr}$:
\[
\mathcal{L}_{\text{TTO}}
= \big\| \mathrm{GSJBU}(I_{lr}) - I_{hr} \big\|_1.
\]
This test-time optimization finds image-specific, pixel-wise kernels that best reconstruct the guidance signal.

After the TTO process, the learned kernels are reused to render the high-resolution feature map:
\[
F_{hr} = \mathrm{GSJBU}(F_{lr};\, \hat{\sigma}_x, \hat{\sigma}_y, \hat{\theta}, \hat{\sigma}_r),
\]
where the optimized parameters $\{\hat{\sigma}_x, \hat{\sigma}_y, \hat{\theta}, \hat{\sigma}_r\}$ 
are directly transferred to the foundation feature space to upsample $F_{lr}$ into $F_{hr}$.

\section{Experiments}

\subsection{Experimental Setting}
Following prior feature-upscaling works~\cite{fu2024featup,suri2024lift,huang2025loftuplearningcoordinatebasedfeature,couairon2025jafar,wimmer2025anyup}, we evaluate Upsample Anything on semantic segmentation (COCO, PASCAL-VOC, ADE20K) and depth estimation (NYUv2). For segmentation, we use a single $1 \times 1$ convolution head (linear probe), identical to prior settings.
For depth, we adopt a DPT-style decoder head, consistent with~\cite{huang2025loftuplearningcoordinatebasedfeature,wimmer2025anyup}. To compare backbones, we consider DINOv1, DINOv2, DINOv3, CLIP, and ConvNeXt, covering both transformer and convolutional families; unless otherwise specified, the default backbone is DINOv2-S.

Unlike prior approaches restricted to feature maps, Upsample Anything applies to general bilateral upsampling. Accordingly, we further evaluate (i) depth-map upsampling on NYUv2 and Middlebury, and (ii) probability-map upsampling on Cityscapes—each guided by the corresponding high-resolution RGB image.
\subsection{Implementation Details}

Our Upsample Anything is implemented purely in PyTorch without any chunked or patch-wise processing. 
All computations are performed in a fully parallel manner over the entire high-resolution grid. 
Gaussian parameters are initialized as $\sigma_x=\sigma_y=16.0$, $\sigma_r=0.12$, and $\theta=0$, 
and are optimized per-pixel using the Adam optimizer with a learning rate of $1\times10^{-3}$. 
The model performs test-time optimization for only 50 iterations in total, without any batching or data augmentation. 
Please refer to the supplementary material for additional implementation details, hyperparameter choices, and ablations.
\begin{table}[t]
\centering
\resizebox{\linewidth}{!}{%
\setlength{\tabcolsep}{6pt}
\renewcommand{\arraystretch}{1.15}
\begin{tabular}{l|cc|cc|cc}
\hline
\multirow{2}{*}{Method} &
\multicolumn{2}{c|}{\textbf{COCO}} &
\multicolumn{2}{c|}{\textbf{PASCAL-VOC}} &
\multicolumn{2}{c}{\textbf{ADE20k}} \\ 
 & mIoU ($\uparrow$) & Acc. ($\uparrow$)
 & mIoU ($\uparrow$) & Acc. ($\uparrow$)
 & mIoU ($\uparrow$) & Acc. ($\uparrow$) \\
\hline
Bilinear & 60.43 & 80.18 & 81.27 & 95.96 & 41.48 & 74.95 \\
FeatUp   & 60.96 & 80.65 & 81.91 & 96.27 & 41.92 & 75.41 \\
LoftUp   & 61.08 & 80.72 & 81.84 & 96.33 & 41.83 & 75.36 \\
JAFAR    & 60.87 & 80.51 & 82.05 & 96.21 & 41.74 & 75.22 \\
AnyUp    & 61.25 & 80.89 & 82.18 & 96.39 & 42.02 & 75.63 \\
\rowcolor[gray]{0.9}
\textbf{Upsample Anything} & \textbf{61.41} & \textbf{81.34} & \textbf{82.22} & \textbf{96.90} & \textbf{42.95} & \textbf{76.52} \\
\hline
\rowcolor[gray]{0.9} \textbf{Upsample Anything (prob.)} & \textbf{63.40} & \textbf{83.73} & \textbf{84.57} & \textbf{97.42} & \textbf{44.29} & \textbf{78.58} \\ \hline
\end{tabular}
} 
\caption{Comparison of different upsampling methods on COCO, PASCAL-VOC, and ADE20k datasets.}
\label{tab:seg_results}
\end{table}

\subsection{quantitative results}
\paragraph{Semantic Segmentation.}
For a fair comparison, we adopt the conventional linear-probe protocol in which prior work fine-tunes only a $1{\times}1$ convolutional head for 10 epochs.  However, we found that this shallow schedule often under-trains the head.  We therefore extend training to 100 epochs and apply a cosine learning-rate schedule to gradually decay the head’s learning rate.  Under this setting, our results in Table~\ref{tab:seg_results} show a trend that differs from previous reports~\cite{fu2024featup,suri2024lift,huang2025loftuplearningcoordinatebasedfeature,couairon2025jafar,wimmer2025anyup}: although all methods converge quickly, their eventual gains over simple bilinear upsampling are modest when the backbone representation is strong. 
This raises the question of how much feature upsampling helps semantic segmentation under high-capacity backbones. 
Nevertheless, our proposed Upsample Anything attains the best accuracy across COCO, PASCAL-VOC, and ADE20K, with AnyUp consistently second.

\textit{Upsample Anything (prob.).} In addition to upsampling feature maps, we evaluate a low-compute variant that \emph{predicts the segmentation at the feature resolution} (no feature upsampling), produces a probabilistic map, and then upsamples \emph{probabilities} to the original image size using our method. 
Because the logits/probabilities live on a much smaller spatial grid, this pipeline achieves the lowest computational cost yet delivers the highest accuracy in Table~\ref{tab:seg_results}. 
This suggests a promising paradigm for segmentation: upsample task probabilities rather than intermediate features.

\begin{table}[t]
\centering
\resizebox{\linewidth}{!}{
\setlength{\tabcolsep}{5pt}
\renewcommand{\arraystretch}{1.2}
\begin{tabular}{l|cc|ccccc}
\toprule
\multirow{2}{*}{\textbf{Method}} 
& \multicolumn{2}{c|}{\textbf{Depth Estimation}} 
& \multicolumn{5}{c}{\textbf{Surface Normal Estimation}} \\ 
\cmidrule(lr){2-3} \cmidrule(lr){4-8}
 & \textbf{RMSE (↓)} & \textbf{$\delta_1$ (↑)} 
 & \textbf{Mean (↓)} & \textbf{Median (↓)} & \textbf{$<\!11.25^\circ$ (↑)} & \textbf{$<\!22.5^\circ$ (↑)} & \textbf{$<\!30^\circ$ (↑)} \\
\midrule
Bilinear & 0.545 & 0.804 & 23.8 & 19.5 & 33.0 & 70.0 & 81.0 \\
FeatUp   & 0.523 & 0.810 & 22.7 & 18.6 & 35.5 & 72.3 & 83.0 \\
LoftUp   & 0.796 & 0.789 & 28.9 & 24.1 & 25.0 & 58.0 & 71.0 \\
JAFAR    & 0.521 & 0.807 & 23.2 & 19.0 & 34.0 & 70.8 & 82.0 \\
AnyUp    & 0.513 & 0.817 & 22.2 & 18.1 & 36.8 & 73.5 & 84.1 \\
\rowcolor{gray!15}
\textbf{Upsample Anything (Ours)} & \textbf{0.498} & \textbf{0.829} & \textbf{21.5} & \textbf{17.4} & \textbf{38.1} & \textbf{75.2} & \textbf{85.8} \\
\bottomrule
\end{tabular}
}
\caption{Comparison of depth and surface normal estimation on the NYUv2 dataset.}
\label{tab:nyu_depth_normal}
\end{table}

\paragraph{Depth Estimation} We evaluate our method on the \textbf{NYUv2} dataset using a frozen DINOv2 backbone. 
Following prior works (AnyUp, LoftUp), we adopt a lightweight DPT-style decoder head for dense prediction (details in Appendix). 
Unlike the original DPT, our Upsample Anything removes the internal interpolation layers, as the feature maps are already upsampled to high resolution. 
As shown in Table~\ref{tab:nyu_depth_normal}, Upsample Anything achieves the best performance on both depth and surface normal estimation 
(RMSE~0.498, $\delta_1$~0.829, mean~21.5°), 
indicating that precise feature upsampling is particularly beneficial for geometry-oriented tasks, 
while LoftUp suffers from domain gaps and fails to generalize. 
It appears that feature upsampling plays a more critical role in depth and surface normal estimation than in semantic segmentation.

\paragraph{Depth Map Upsampling}
Unlike feature upsampling tasks, our Upsample Anything can also be applied to \emph{other modalities} such as raw depth maps. 
In Table~\ref{tab:depth_up_results}, we evaluate Upsample Anything by downsampling high-resolution depth maps to $32{\times}32$ and restoring them to $512{\times}512$ resolution.
This setup shares the same bilateral upsampling pipeline as our feature experiments, except that the low-resolution input is a depth map itself.
We compare against the state-of-the-art guided interpolation method (GLU)~\cite{song2023guided} and the bilinear baseline.
As shown in Figure~\ref{fig:depth}, Upsample Anything achieves the best performance on the Middlebury dataset, producing sharper and more consistent structures.
In contrast, on the NYUv2 dataset, the bilinear method yields a slightly lower RMSE (0.159 vs. 0.237), likely because the ground-truth depth maps are blurred and contain smoother structures.
Nevertheless, the qualitative results suggest that Upsample Anything preserves geometry more effectively, especially for high-frequency and edge-dominant regions.

\begin{table}[t]
\centering
\resizebox{1.0\linewidth}{!}{
\setlength{\tabcolsep}{5pt}
\renewcommand{\arraystretch}{1.2}
\begin{tabular}{l|cc|cc}
\toprule
\multirow{2}{*}{\textbf{Method}} 
& \multicolumn{2}{c|}{\textbf{NYUv2 (Depth Up.)}} 
& \multicolumn{2}{c}{\textbf{Middlebury (Depth Up.)}} \\ 
\cmidrule(lr){2-3} \cmidrule(lr){4-5}
& \textbf{RMSE (↓)} & \textbf{$\delta_1$ (↑)} 
& \textbf{RMSE (↓)} & \textbf{$\delta_1$ (↑)} \\
\midrule
Bilinear & 0.167 & 0.983 & 0.231 & 0.962 \\
GLU (Guided Linear Upsample) & 0.372 & 0.841 & 0.491 & 0.825 \\
\rowcolor{gray!15}\textbf{Upsample Anything (Ours)} & 0.214 & 0.976 & \textbf{0.209} & \textbf{0.967} \\
\bottomrule
\end{tabular}
}
\caption{Comparison on NYUv2 and Middlebury depth upsampling.}
\label{tab:depth_up_results}
\end{table}

\begin{figure}[t!]
    \centering
    \includegraphics[width=0.9\columnwidth]{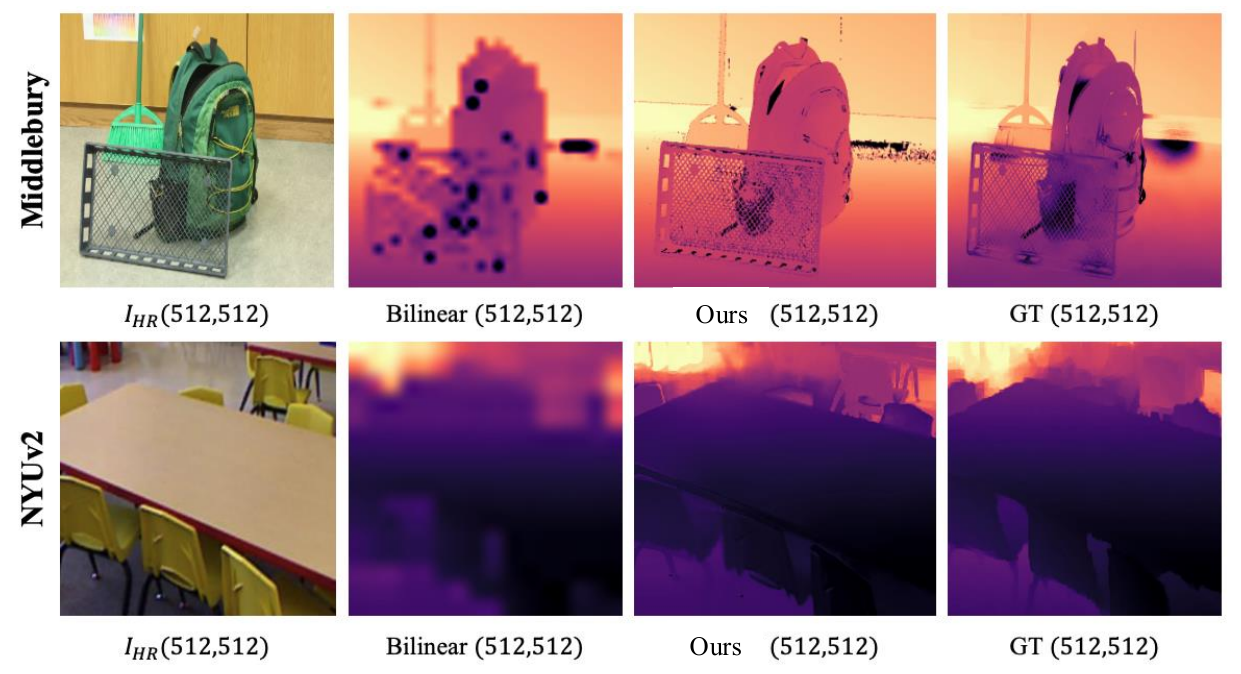}
    \caption{Depth upsampling results on Middlebury (top) and NYUv2 (bottom).  32×32 low-resolution depth maps were upsampled to high resolution using different methods. 
While Upsample Anything produces sharper and more detailed edges, it still achieves lower RMSE (0.237) than bilinear (0.159) on low-resolution maps. 
However, in high-resolution depth prediction, Upsample Anything outperforms both qualitatively and quantitatively.}
\label{fig:depth}
\end{figure}

\begin{figure*}[t!]
    \centering
    \includegraphics[width=1.7\columnwidth]{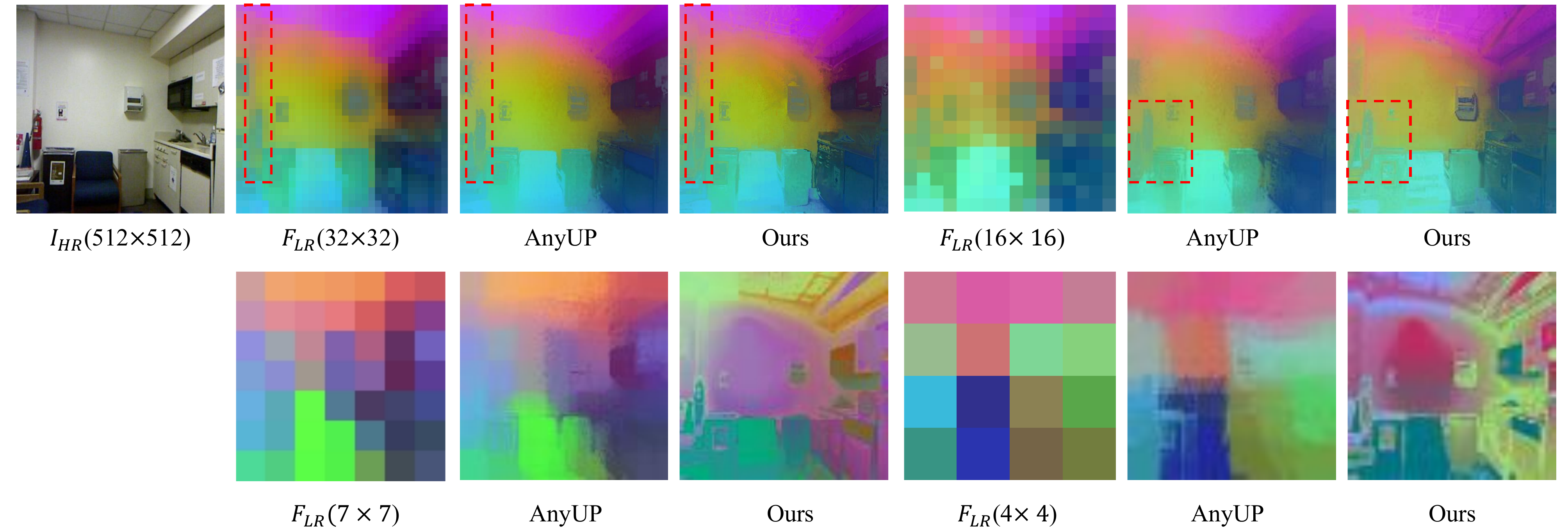}
    \caption{Comparison across different resolutions. Qualitative results of AnyUp (previous SOTA) and our Upsample Anything on varying input resolutions.}

\label{fig:different_resolution}
\end{figure*}

\subsection{Qualitative results}
\paragraph{Across Different Resolutions.}
Figure \ref{fig:different_resolution} compares AnyUp (previous SOTA) and our Upsample Anything across different input resolutions.
As shown, AnyUp performs reasonably well at higher resolutions (e.g., $32{\times}32$ and $16{\times}16$) but tends to produce over-smoothed regions, as highlighted by the red boxes.
In contrast, Upsample Anything maintains sharp boundaries and fine structures even at extremely low resolutions (e.g., $7{\times}7$ and $4{\times}4$), demonstrating stronger robustness to spatial degradation.
\paragraph{Across Different Backbones.}
We compare the visual quality of upsampled features produced by AnyUP and our Upsample Anything across various backbone architectures.
Given the 224×224 input image, the spatial resolutions of the extracted feature maps differ by model: 7×7 for ConvNeXt, 14×14 for CLIP and DINOv1, and 16×16 for DINOv2 and DINOv3.
In this example, $\hat{I}_{HR}$ denotes the reconstructed high-resolution image obtained from a 7×7 low-resolution $I_{LR}$ using Upsample Anything.
As shown in ~\cref{fig:different_backbone}, Upsample Anything consistently produces sharper boundaries, finer local structures, and more coherent feature clustering than AnyUP across all backbones.
We attribute this advantage to Upsample Anything’s test-time optimization, which adaptively fits Gaussian parameters to each input image, yielding features that align more precisely with image-level semantics.
\paragraph{Feature Similarity Analysis.}
~\cref{fig:cossim} visualizes the feature similarity between two different images that share the same object category.
This setting is often adopted in few-shot segmentation tasks to assess the consistency of feature representations.
As shown in the figure, both AnyUp and Upsample Anything produce visually plausible feature upsampling results; however, when measuring cosine similarity, AnyUp tends to yield uniformly high similarity across the entire image, indicating a lack of spatial discrimination.
In contrast, our Upsample Anything produces sharper and more localized similarity maps,
where object boundaries are clearly preserved and distinct regions are well separated.
We believe that such discriminative feature behavior suggests the potential of our method for downstream tasks such as few-shot segmentation and category-level feature matching.

\begin{figure*}[t!]
    \centering
    \includegraphics[width=1.8\columnwidth]{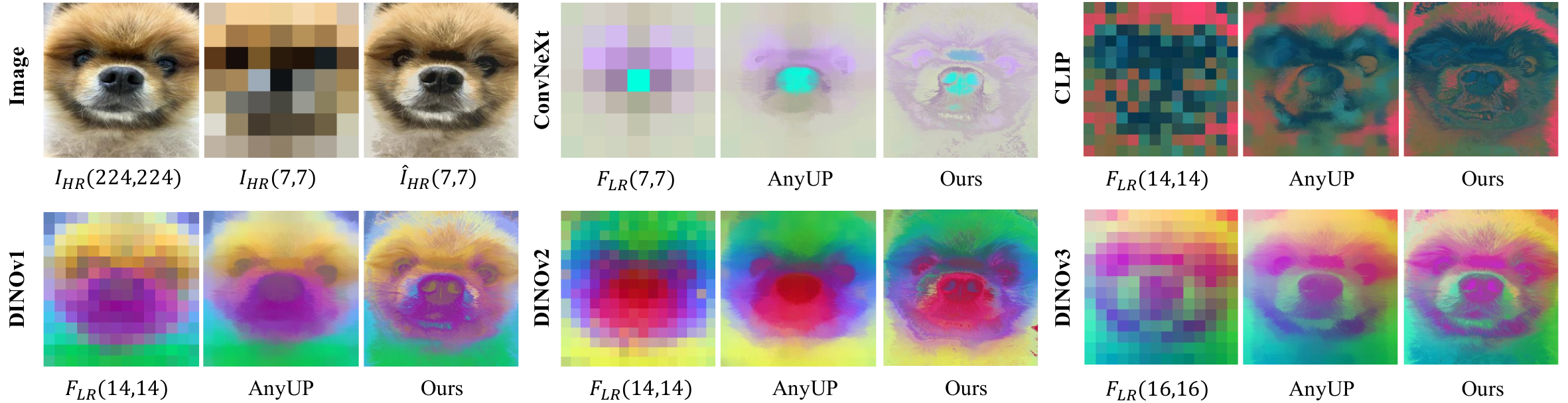}
    \caption{Visual comparison across different backbones. Given the same 224×224 input, feature maps have varying spatial resolutions (7×7 for ConvNeXt, 14×14 for CLIP and DINOv1, 16×16 for DINOv2/v3). Upsample Anything produces sharper edges, richer textures, and more distinct feature clustering than AnyUP across all backbones, demonstrating its strong adaptability through test-time optimization.}
\label{fig:different_backbone}
\end{figure*}

\begin{figure}[t!]
    \centering
    \includegraphics[width=1.0\columnwidth]{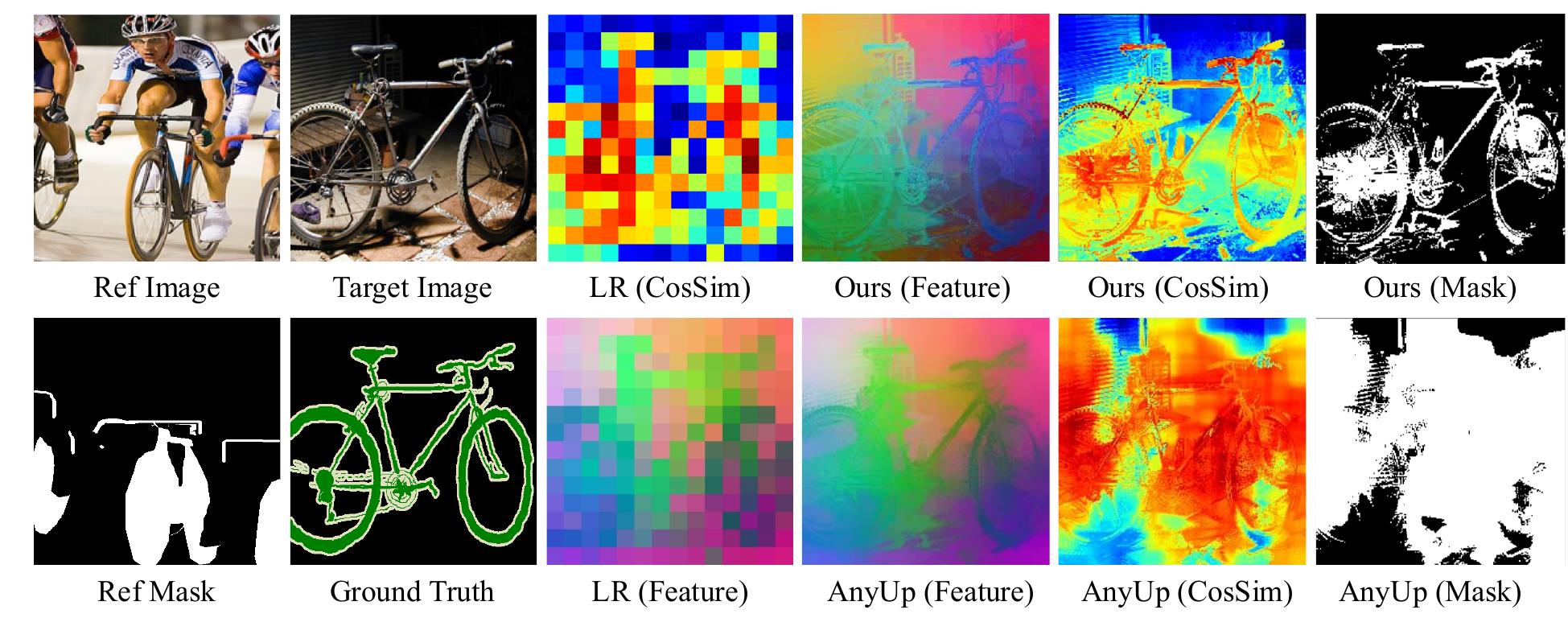}
    \caption{Visualization of feature similarity between a reference and a target image. The feature vector is obtained by averaging the reference features within the reference mask, and cosine similarity is then computed against all feature locations in the target image.}
\label{fig:cossim}
\end{figure}
\subsection{Ablation Study}
\paragraph{Resolution–Efficiency Tradeoff} We analyze the computational efficiency of our Upsample Anything compared with AnyUp~\cite{wimmer2025anyup} across multiple output resolutions, as summarized in Table~\ref{tab:anyup_gsjbu_efficiency}. 
AnyUp employs a feature-agnostic layer and local window attention that allow flexible inference across arbitrary encoders and resolutions. 
While this design yields fast performance at lower resolutions ($<\!256\times256$), the window-based attention and dense similarity computation introduce quadratic growth in both memory and time complexity as the spatial size increases. 
Consequently, AnyUp suffers from significant GPU memory overhead and fails with out-of-memory (OOM) errors beyond $512\times512$, exposing a scalability bottleneck in its dense attention formulation.

In contrast, our Upsample Anything performs fully parallel anisotropic Gaussian splatting without building dense pairwise affinity maps, resulting in linear memory growth with respect to the output size. 
As shown in Table~\ref{tab:anyup_gsjbu_efficiency}, Upsample Anything maintains stable runtime and controlled memory even at $1024\times1024$ resolution—where AnyUp cannot execute—demonstrating its robustness for large-scale feature upsampling. 

\begin{table}[t]
\centering
\resizebox{0.9\linewidth}{!}{
\setlength{\tabcolsep}{5pt}
\renewcommand{\arraystretch}{1.2}
\begin{tabular}{c|c|c}
\toprule
\textbf{Resolution (H×W)} & \textbf{AnyUp (trained) /}  & \textbf{AnyUp (trained) /} \\
 & \textbf{Ours (TTA) Time (s)} &\textbf{Ours (TTA) Peak Mem (MB)} \\
\midrule
64×64    & 0.0025 / 0.0055 & \textbf{53.8} / 358.5 \\
128×128  & 0.0034 / 0.0150 & \textbf{121.9} / 1320.5 \\
224×224  & 0.0137 / 0.0419 & \textbf{531.0} / 3969.7 \\
448×448  & 0.0583 / 0.2398 & \textbf{6184.2} / 15774.8 \\
512×512  & 0.0893 / 0.3211 & 10283.8 / \textbf{20590.4} \\
896×896  & 0.5875 / 1.2789 & 91250.9 / \textbf{62985.2} \\
1024×1024 & OOM / 1.8083   & — / \textbf{82255.5} \\
2048×2048 & OOM / OOM      & — / — \\
\bottomrule
\end{tabular}
}
\caption{omparison of inference time and GPU memory usage between AnyUp and \textbf{Upsample Anything} under varying input resolutions.  Both methods operate without training, but \textbf{Upsample Anything} performs per-image test-time optimization (TTO), which results in slightly higher computational cost but significantly improved generalization.}
\label{tab:anyup_gsjbu_efficiency}
\end{table}

\paragraph{Why Upsample Anything: Design Motivation and Comparative Analysis} In our architecture search, we aimed to achieve accurate upsampling within one second while maintaining generalization across domains.
~\cref{tab:time_seg_depth} compares different upsampling strategies on PASCAL-VOC and NYUv2 benchmarks. Guided Linear Upsampling (GLU)~\cite{song2023guided} represents a strong bilateral-filter-based baseline but exhibits unstable optimization due to its unconstrained formulation.
We also implemented a 2D Gaussian Splatting (2DGS)~\cite{huang20242d} baseline by directly optimizing Gaussian kernels for low-to-high feature interpolation without hierarchical fitting.
Although 2DGS can model local structures continuously, its dense Gaussian representation and lack of spatial–range constraints make it computationally heavy and prone to over-smoothing during test-time optimization.
In contrast, Upsample Anything inherits the spatial–range constraint of classical JBU while leveraging the continuous Gaussian formulation for differentiable optimization.
As shown in ~\cref{tab:time_seg_depth} Upsample Anything achieves the best balance between speed, convergence, and accuracy, demonstrating that the JBU constraint synergizes effectively with Gaussian optimization for fast and stable test-time refinement.

\begin{table}[t]
\centering
\resizebox{\linewidth}{!}{
\setlength{\tabcolsep}{6pt}
\renewcommand{\arraystretch}{1.2}
\begin{tabular}{l|c|cc|cc}
\toprule
\multirow{2}{*}{\textbf{Method}} 
& \multirow{2}{*}{\textbf{Time (s)}} 
& \multicolumn{2}{c|}{\textbf{PASCAL-VOC (Segmentation)}} 
& \multicolumn{2}{c}{\textbf{NYUv2 (Depth Estimation)}} \\ 
\cmidrule(lr){3-4} \cmidrule(lr){5-6}
 & & \textbf{mIoU (↑)} & \textbf{Acc. (↑)} & \textbf{RMSE (↓)} & \textbf{$\delta_1$ (↑)} \\
\midrule
Bilinear & 0.00009 & 81.27 & 95.96 & 0.545 & 0.804 \\
Guided Linear Upsample & 0.00303 & 80.12 & 94.73 & 0.598 & 0.773  \\
JBU  & 0.00600 &81.65 & 96.21 & 0.531 & 0.812 \\
LIG & 481.526 &78.54 & 93.02 & 0.642 & 0.741 \\
\rowcolor{gray!15}
\textbf{GSJBU (Ours)} & 0.4197 & \textbf{82.22} & \textbf{96.90} & \textbf{0.498} & \textbf{0.829} \\
\bottomrule
\end{tabular}
}
\caption{
Comparison of inference time, segmentation, and depth estimation performance across different upsampling methods. 
GSJBU achieves the best balance between accuracy and scalability.
}
\label{tab:time_seg_depth}
\end{table}

\paragraph{Impact of Test-Time Optimization Steps}
We conducted experiments to investigate the trade-off between the number of TTO iterations, inference time, and performance.
Table~\ref{tab:iteration_ablation} summarizes the results.
We measured PSNR between $I_{lr}$ and $\hat{I}_{hr}$, along with downstream segmentation performance on PASCAL-VOC.
As shown, PSNR quickly converges to 35.60 after about 500 iterations, indicating that Upsample Anything reaches its optimum very early.
Interestingly, the best segmentation accuracy is achieved at only 50 iterations, which also provides the fastest inference time (0.419s).
Based on this observation, we adopt 50 iterations as the default setting throughout all experiments.

\begin{table}[t]
\centering
\resizebox{0.9\linewidth}{!}{
\setlength{\tabcolsep}{6pt}
\renewcommand{\arraystretch}{1.2}
\begin{tabular}{c|c|cc|c}
\toprule
\multirow{2}{*}{\textbf{Iteration}} 
& \multirow{2}{*}{\textbf{PSNR (↑)}} 
& \multicolumn{2}{c|}{\textbf{PASCAL-VOC (Segmentation)}} 
& \multirow{2}{*}{\textbf{Time (s)}} \\ 
\cmidrule(lr){3-4}
 & & \textbf{mIoU (↑)} & \textbf{Acc. (↑)} & \\
\midrule
\rowcolor{gray!15}50    & 35.33 & \textbf{82.22} & \textbf{96.90} & \textbf{0.041} \\
300   & 35.59 & 82.10 & 96.84 & 3.397 \\
500   & \textbf{35.60} & 82.15 & 96.88 & 6.161 \\
1000  & \textbf{35.60} & 82.10 & 96.82 & 12.294 \\
5000  & \textbf{35.60} & 82.17 & 96.80 & 61.458 \\
\bottomrule
\end{tabular}
}
\caption{
Ablation on the number of optimization iterations. }
\label{tab:iteration_ablation}
\end{table}

%
%
%
%
%
\section{Conclusion}
We introduced Upsample Anything, a unified framework that connects Joint Bilateral Upsampling (JBU) and Gaussian Splatting (GS) under a continuous formulation. 
It performs lightweight test-time optimization without pre-training or architectural constraints, achieving efficient and robust upsampling across diverse resolutions and domains. 
It optimizes a 224×224 image in 0.419 seconds while producing significant gains in both feature and depth upsampling. 
Extensive experiments show that Upsample Anything achieves state-of-the-art performance without any learnable module, serving as a universal, plug-and-play framework that combines the simplicity of JBU with the expressive power of Gaussian representation. 
\noindent\textbf{Limitation.} 
Despite its generality, \textit{Upsample Anything} may face challenges under severe occlusions or low-SNR guidance, where optimization becomes unstable. 
Future work will focus on enhancing the robustness and adaptability of the framework across diverse domains, aiming to make it more resilient under challenging conditions.

\clearpage
\setcounter{page}{1}
\maketitlesupplementary

\section{Semantic Segmentation on Cityscapes}
We evaluated feature upsampling and probability-map upsampling on Cityscapes using the official LoftUp segmentation codebase with a stronger training setup that includes 448×448 input resolution, 100 epochs, and a learning-rate scheduler.
Under this configuration, all methods, including LoftUp and ours, produced almost the same mIoU as bilinear interpolation, which differs from the improvements reported in the LoftUp paper.

To ensure correctness, we carefully re-examined our implementation through an automated code audit with ChatGPT and a manual review by multiple authors. We found no inconsistencies or bugs.

The quantitative results are summarized in Table~\ref{tab:cityscapes_results}. 
Across all methods, including feature-level and probability-level upsampling, the differences remain within a very narrow range. 
Cityscapes primarily contains large and regular structures, and its annotations are relatively coarse. 
With a sufficiently trained segmentation head, bilinear interpolation already performs near optimally, leaving little room for additional gains. 
In contrast, datasets such as COCO, PASCAL-VOC, and ADE20K include many small objects and complex boundaries, where upsampling delivers clear benefits.

\begin{table}[h]
\centering
\resizebox{\linewidth}{!}{%
\setlength{\tabcolsep}{6pt}
\renewcommand{\arraystretch}{1.15}
\begin{tabular}{l|cc}
\hline
\multirow{2}{*}{Method} &
\multicolumn{2}{c}{\textbf{Cityscapes}} \\
 & mIoU ($\uparrow$) & Acc. ($\uparrow$) \\
\hline
Bilinear & 57.90 & 90.59 \\
FeatUp   & 57.92 & 90.61 \\
LoftUp   & 57.89 & 90.60 \\
JAFAR    & 57.91 & 90.58 \\
AnyUp    & 57.93 & 90.62 \\
Upsample Anything & 57.92 & 90.63 \\
\hline
Upsample Anything (prob.) & 56.36 & 90.05 \\
\hline
\end{tabular}
} 
\caption{Segmentation performance on the Cityscapes dataset using the official LoftUp evaluation pipeline.}
\label{tab:cityscapes_results}
\end{table}

\section{Details of Probabilistic Map Upsampling in Table 1.}
Figure~\ref{fig:segseg}-(c) corresponds to the \textit{Upsample Anything (prob.)} configuration reported in Table~1. 
In this setting, the segmentation map is predicted from downsampled features using a lightweight 1×1 convolution, followed by our probabilistic upsampling to reconstruct high-resolution outputs.
This simple setup already shows strong performance.
Because the computation is performed on a small feature map, heavier or more complex decoders are expected to be feasible without large computational overhead.
This suggests that the proposed segmentation pipeline has further potential, although designing such decoders is beyond the scope of this work.

\begin{figure}[t!]
    \centering
    \includegraphics[width=1.0\columnwidth]{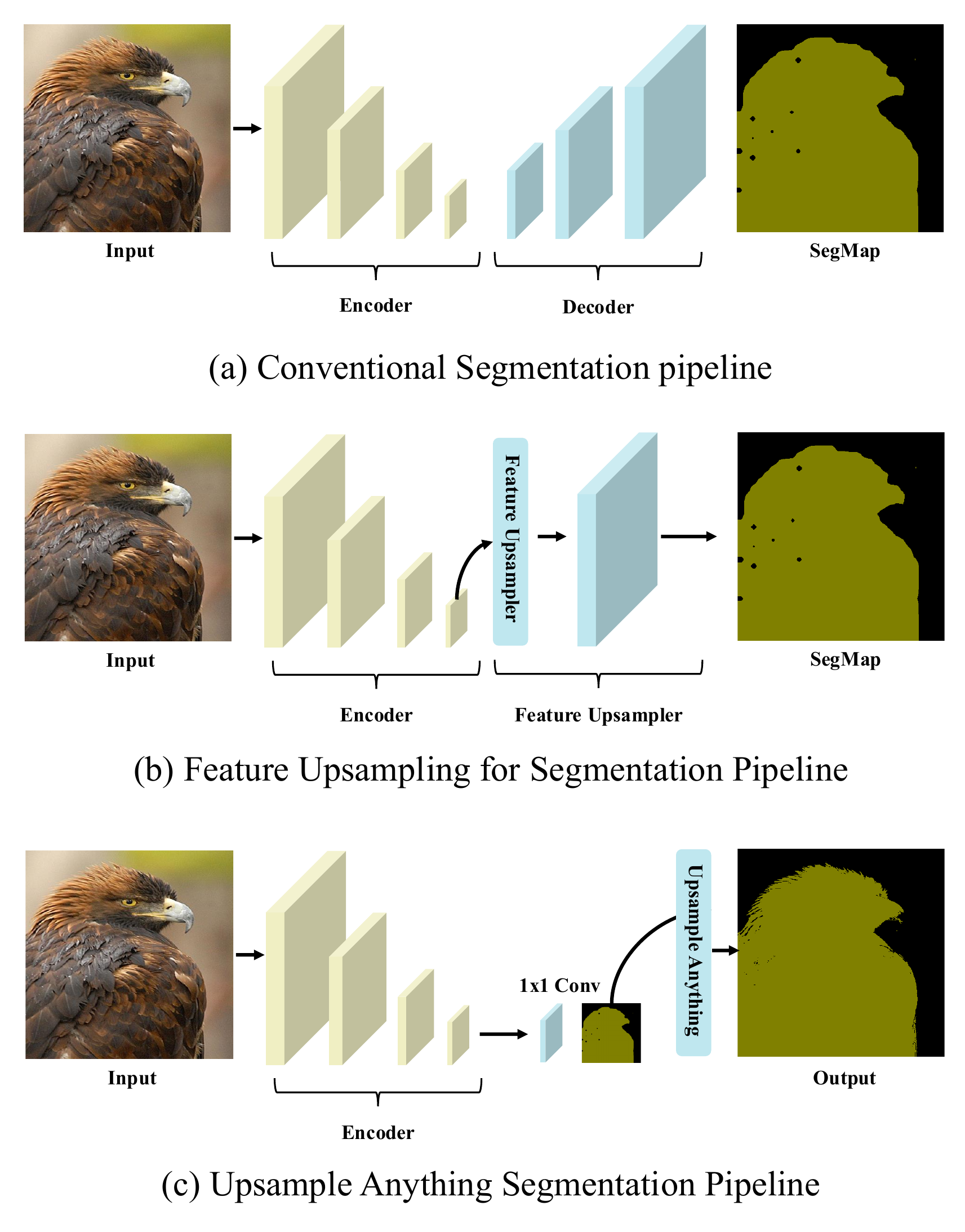}
    \caption{Comparison of segmentation pipelines. (a) Conventional segmentation pipeline with a Vision Foundation Model encoder and task-specific decoders such as DPT, UPerNet, SegFormer, or Mask2Former. (b) Feature upsampling pipeline using pretrained upsamplers such as FeatUP, LoftUP, JAFAR, or AnyUp, operating on feature maps. (c) Our proposed Upsample Anything, which performs test-time optimization and handles both feature and segmentation upsampling without additional training.}
\label{fig:segseg}
\end{figure}

\section{Details of Depth Estimation in Table 2.}
We followed the DPT-based depth estimation setup used in prior works~\cite{huang2025loftuplearningcoordinatebasedfeature, wimmer2025anyup}. 
Although DPT includes an internal upsampling, its exact implementation details are not provided in the paper or codebase. 
Therefore, we reimplemented the head as described in Algorithm~\ref{alg:dptdepth} and used it consistently for all our depth estimation experiments.

\begin{algorithm}[t!]
\caption{Depth estimation setting used in Table~2.}
\label{alg:dptdepth}
\begin{algorithmic}[1]
\REQUIRE Input image $x$
\ENSURE Predicted depth map $D$
\STATE Extract multi-scale features $F$ using a pretrained Vision Foundation Model encoder (e.g., DINOv2)
\STATE Pass $F$ through the DPT head:
\STATE \hspace{1em}$F_1 \leftarrow \mathrm{Conv}(F,\, k{=}3,\, c{\rightarrow}c/2)$
\STATE \hspace{1em}$F_2 \leftarrow \mathrm{Conv}(F_1,\, k{=}3,\, c/2{\rightarrow}32)$
\STATE \hspace{1em}$F_3 \leftarrow \mathrm{ReLU}(F_2)$
\STATE \hspace{1em}$D_{inv} \leftarrow \mathrm{Conv}(F_3,\, k{=}1,\, 32{\rightarrow}1)$
\STATE \hspace{1em}$D_{inv} \leftarrow \mathrm{ReLU}(D_{inv})$ (if non-negative)
\IF{\texttt{invert} is True}
    \STATE $D \leftarrow \frac{1}{\mathrm{clip}(s \cdot D_{inv} + t,\, 1\mathrm{e}{-8},\, \infty)}$
\ELSE
    \STATE $D \leftarrow D_{inv}$
\ENDIF
\STATE Output the final depth prediction $D$
\end{algorithmic}
\end{algorithm}

\section{From 2D Low-Resolution Feature Maps to 3D High-Resolution Feature Volumes}
\label{sec:3d_upsampling}

\begin{figure*}[t!]
    \centering
    \includegraphics[width=1.8\columnwidth]{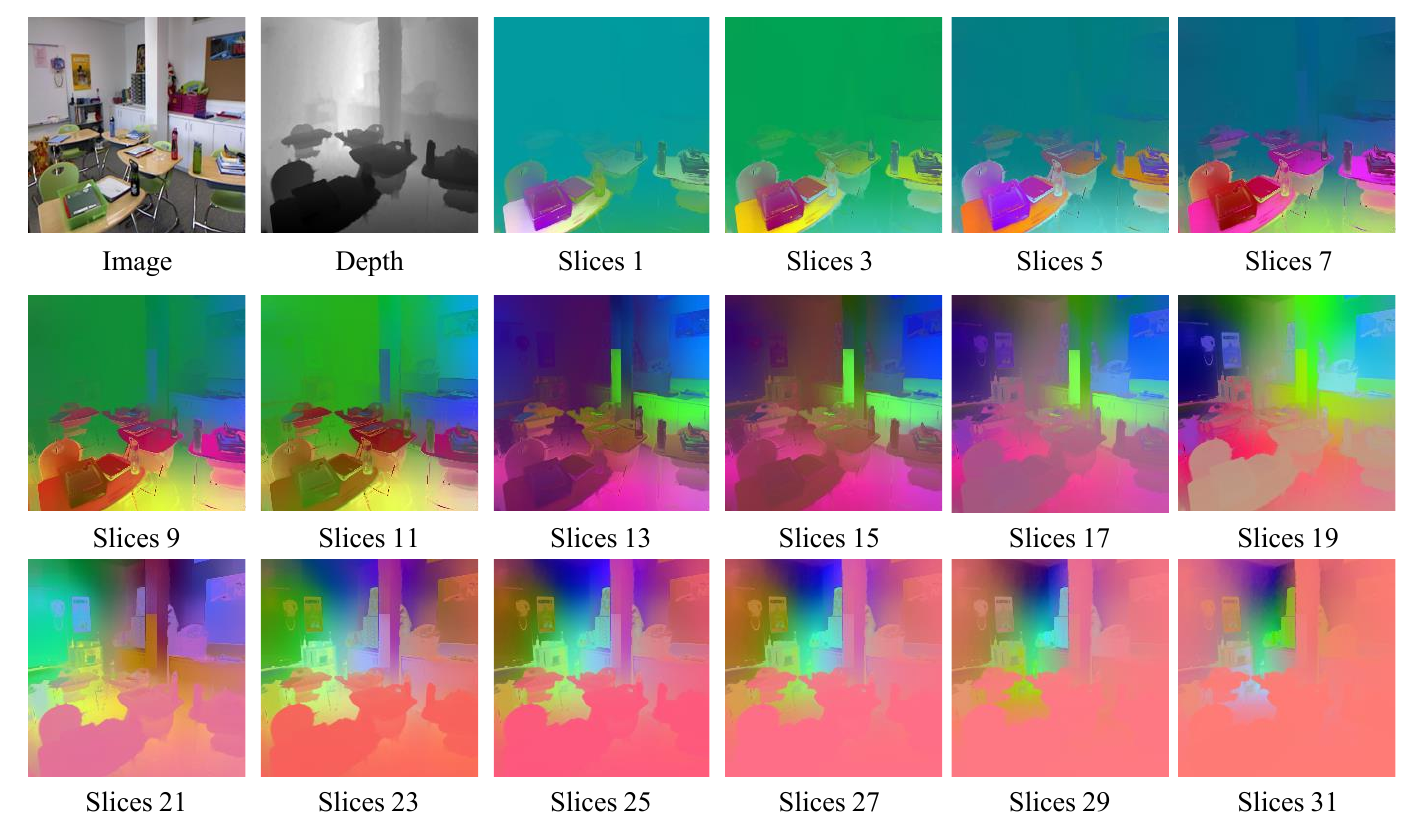}
    \caption{Visualization of our 3D feature upsampling results. The first two columns show the RGB image and its corresponding depth map. The remaining panels depict representative depth slices from the reconstructed 3D high-resolution feature volume obtained using our Upsample Anything. Each slice is visualized via PCA projection into RGB space. Notice that the recovered 3D feature layers exhibit smooth transitions along the depth axis while preserving fine object boundaries and geometric continuity.}
\label{fig:feat3d_grid}
\end{figure*}

We extend our test-time optimization (TTO) framework to reconstruct dense 3D feature volumes directly from low-resolution 2D feature maps.
Starting with an RGB–Depth (RGB-D) pair, we first downsample the RGB image by a factor of $s$ to simulate a low-resolution feature space.
The corresponding high-resolution RGB-D map is used as the guide signal for optimization. 
During TTO stage, we train only the pixel-wise anisotropic Gaussian kernel parameters $(\sigma_x, \sigma_y, \sigma_z, \theta, \sigma_r)$ so that the 3D Upsample Anyting can accurately project the low-resolution RGB features to their high-resolution RGB-D counterparts.
Once optimized, these learned kernels are frozen and reused in upsample stage to upsample semantic features extracted from 2D LR feature into full 3D feature volumes.
This process allows each low-resolution feature token to be expanded not only spatially along the $x$–$y$ plane, but also along the depth axis $z$, guided by the HR depth map.
The resulting tensor $\mathbf{F}_{3D} \in \mathbb{R}^{D_h \times C \times H_h \times W_h}$ captures the local geometric and appearance-aware structure of the scene.
We visualize these 3D feature maps using PCA on each depth slice, revealing how distinct depth layers retain meaningful semantic separation while smoothly transitioning across depth.
Figure~\ref{fig:feat3d_grid} shows that even without explicit 3D supervision, our Full3DJBU reconstructs volumetric features that align with depth continuity, edges, and object boundaries—demonstrating that our framework can generalize from 2D low-resolution feature inputs to 3D high-resolution representations at test time.

\begin{figure}[t!]
    \centering
    \includegraphics[width=1.0\columnwidth]{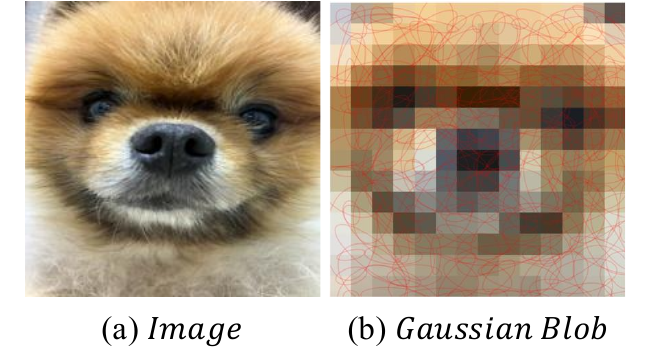}
    \caption{
    Visualization of learned Gaussian blobs.
    (a) shows the original image and (b) displays the Gaussian blobs overlaid on the low-resolution input.
    The blobs reveal locally coherent directions and magnitudes, indicating that the learned kernels adapt to the underlying structure of the scene.
    }
\label{fig:gs_vis}
\end{figure}

\begin{figure*}[t!]
    \centering
    \includegraphics[width=2.0\columnwidth]{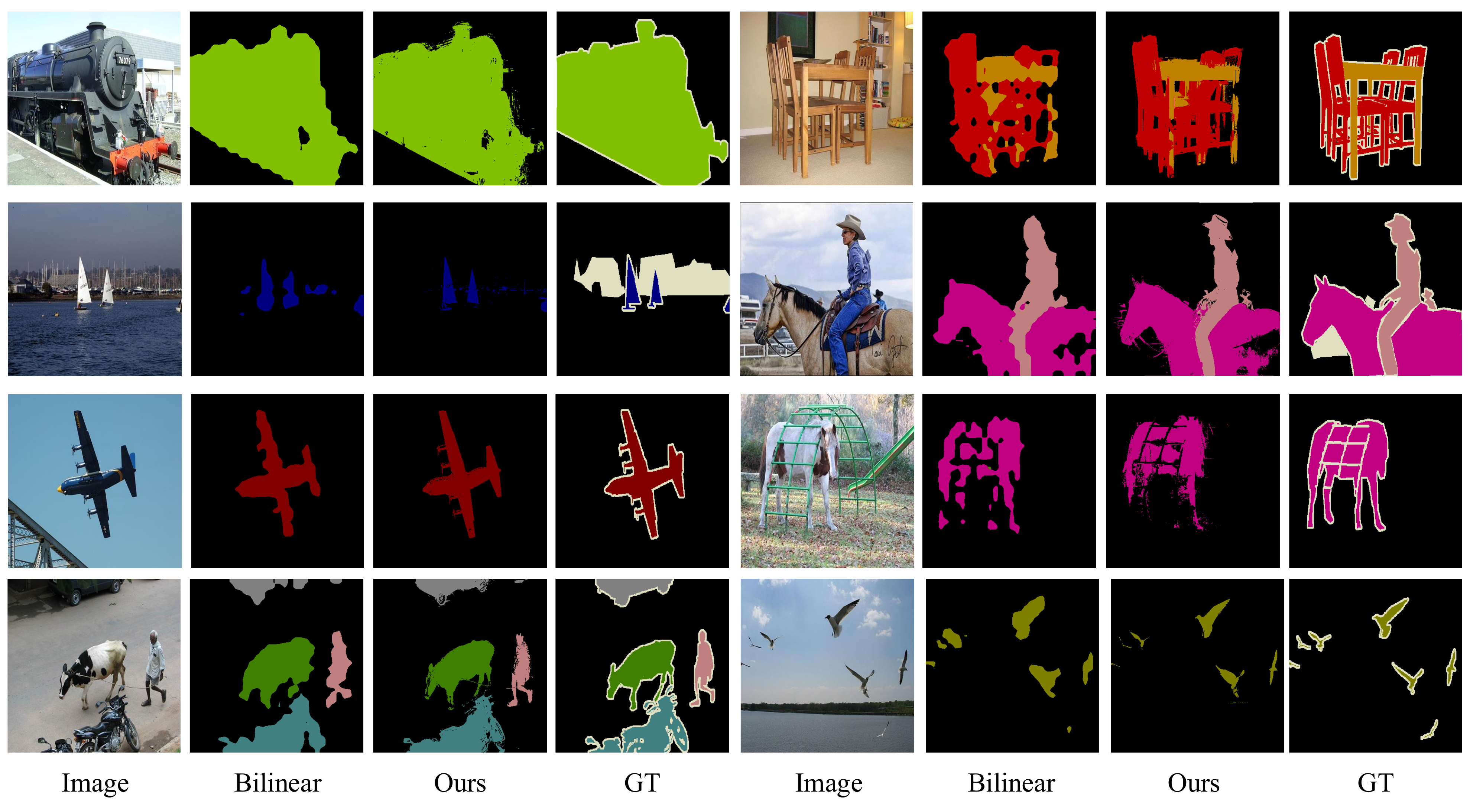}
    \caption{
    Visualization of the \textit{segment-then-upsample} pipeline.
    The segmentation logits are first generated at low resolution and then upsampled by $16\times$ using our method.
    The results exhibit remarkably sharp object boundaries and preserve semantic coherence, highlighting the effectiveness of our upsampling approach.
    }
\label{fig:seg_upsample_vis}
\end{figure*}

\section{Gaussian Blob Visualization}
To better understand what our learned anisotropic kernels capture, we visualized the Gaussian blobs of our model, as shown in Fig.~\ref{fig:gs_vis}.
(a) shows the original high-resolution RGB image, and (b) overlays the learned Gaussian blobs on the corresponding low-resolution image.
Although the visualization can be difficult to interpret directly, certain spatial regions such as the eyes, nose, and corners exhibit overlapping or consistently oriented blobs, which suggests that nearby kernels capture semantically similar local structures.
This indicates that the learned kernels adaptively encode meaningful directional features rather than behaving randomly.
However, similar to other methods that rely on Gaussian Splatting or 2DGS, not every blob is fully interpretable, and some visual noise appears due to overparameterization and kernel redundancy.

\section{Segment-then-Upsample Pipeline Visualization Results}
This section presents the visualization results of our \textit{segment-then-upsample} pipeline, corresponding to the method in Fig.~\ref{fig:segseg}-(c). 
In this configuration, we perform semantic segmentation on the low-resolution feature maps first and subsequently upsample the segmentation logits by a factor of $16\times$. 
Despite the large upsampling ratio, our Upsample Anything produces visually sharp and semantically consistent results, as illustrated in Fig.~\ref{fig:seg_upsample_vis}. 
Compared to conventional bilinear interpolation, the recovered boundaries and fine structures are significantly clearer.

\section{Formal Relation Between Joint Bilateral Upsampling and Gaussian Splatting}
\label{sec:jbu_gs_formulation}
The purpose of this section is not to claim that Joint Bilateral Upsampling (JBU) and Gaussian Splatting (GS) are mathematically equivalent.
Instead, we aim to show why the GS framework provides a useful foundation for our formulation.
By reinterpreting JBU through the perspective of GS, we reveal a common idea based on continuous and differentiable Gaussian kernels, which motivates our use of GS-style parameter learning in the \textbf{Upsample Anything (GSJBU)} framework.
In short, this section clarifies the conceptual link between the two views and explains why GS-based test-time optimization naturally applies to feature upsampling.

\paragraph{Notation.}
Let $F_{\mathrm{lr}}:\mathcal{Q}\!\to\!\mathbb{R}^{C}$ be a low-resolution feature map on a discrete grid $\mathcal{Q}\subset\mathbb{Z}^2$,
and let $I:\Omega\!\to\!\mathbb{R}^{d}$ be an HR guidance signal ($d{=}1$ for grayscale, $d{=}3$ for RGB, etc.).
For $p\in\Omega\subset\mathbb{R}^2$, classical JBU is
\begin{equation}
\small
\hat{F}_{\mathrm{hr}}(p)=
\frac{
\sum\limits_{q\in\Omega(p)}
F_{\mathrm{lr}}(q)\,
\exp\!\big(-\tfrac{\|p-q\|^2}{2\sigma_s^2}\big)\,
\exp\!\big(-\tfrac{\|I(p)-I(q)\|^2}{2\sigma_r^2}\big)}
{\sum\limits_{q\in\Omega(p)}
\exp\!\big(-\tfrac{\|p-q\|^2}{2\sigma_s^2}\big)\,
\exp\!\big(-\tfrac{\|I(p)-I(q)\|^2}{2\sigma_r^2}\big)}.
\label{eq:jbu_appen}
\end{equation}

\vspace{2pt}
\noindent\textbf{Joint spatial–range lifting.}
Define the lifted embedding
\[
\phi:\Omega\to\mathbb{R}^{2+d},\qquad
\phi(x) := \begin{bmatrix}x\\ I(x)\end{bmatrix},
\]
and the block-diagonal covariance
\[
\Lambda(\sigma_s,\sigma_r) :=
\mathrm{diag}\!\big(\sigma_s^2 I_2,\ \sigma_r^2 I_d\big)\in\mathbb{R}^{(2+d)\times(2+d)}.
\]
For $u,v\in\mathbb{R}^{2+d}$, let
\[
\mathcal{G}_{\Lambda}(u,v):=
\exp\!\Big(-\tfrac{1}{2}\,(u-v)^{\!\top}\Lambda^{-1}(u-v)\Big).
\]

\begin{theorem}[JBU as a normalized Gaussian mixture in the joint domain]
\label{thm:jbu_as_gs_joint}
Fix $\sigma_s{>}0,\ \sigma_r{>}0$ and let $\Lambda=\Lambda(\sigma_s,\sigma_r)$.
Then for any $p\in\Omega$,
\begin{equation}
\small
\hat{F}_{\mathrm{hr}}(p)=
\frac{\sum\limits_{q\in\Omega(p)} F_{\mathrm{lr}}(q)\,
\mathcal{G}_{\Lambda}\!\big(\phi(p),\,\phi(q)\big)}
{\sum\limits_{q\in\Omega(p)} \mathcal{G}_{\Lambda}\!\big(\phi(p),\,\phi(q)\big)}.
\label{eq:jbu_joint_form}
\end{equation}
In particular, JBU coincides with evaluating a \emph{normalized} Gaussian mixture in the lifted space $\mathbb{R}^{2+d}$ whose centers are $\{\phi(q)\}_{q\in\Omega(p)}$ and whose (isotropic-by-block) covariance is $\Lambda$.
\end{theorem}

\begin{proof}
\small
By construction,
\(
\|\phi(p)-\phi(q)\|_{\Lambda^{-1}}^2
=
(p-q)^{\!\top}(\sigma_s^{-2}I_2)(p-q)+
(I(p)-I(q))^{\!\top}(\sigma_r^{-2}I_d)(I(p)-I(q)).
\)
Thus
\(
\mathcal{G}_{\Lambda}(\phi(p),\phi(q))
=
\exp\!\big(-\tfrac{\|p-q\|^2}{2\sigma_s^2}\big)\,
\exp\!\big(-\tfrac{\|I(p)-I(q)\|^2}{2\sigma_r^2}\big),
\)
and substituting this identity in \eqref{eq:jbu} yields \eqref{eq:jbu_joint_form}.
\end{proof}

\begin{corollary}[Discrete GS view in the joint domain]
\label{cor:jbu_is_discrete_gs}
Let $\mu_q:= \phi(q)$ and $f_q:= F_{\mathrm{lr}}(q)$.
Then Theorem~\ref{thm:jbu_as_gs_joint} states that JBU equals
\begin{equation}
\small
\hat{F}_{\mathrm{hr}}(p)=
\frac{\sum\limits_{q} f_q\, \exp\!\big(-\tfrac{1}{2}(\phi(p)-\mu_q)^{\!\top}\Lambda^{-1}(\phi(p)-\mu_q)\big)}
{\sum\limits_{q} \exp\!\big(-\tfrac{1}{2}(\phi(p)-\mu_q)^{\!\top}\Lambda^{-1}(\phi(p)-\mu_q)\big)},
\end{equation}
i.e., a \emph{Gaussian Splatting} evaluation in $\mathbb{R}^{2+d}$ with fixed block-diagonal covariance and centers on the lifted LR grid.
\end{corollary}

\paragraph{Connection to standard 2D GS.}
Standard (2D) GS writes, for $p\in\mathbb{R}^2$,
\begin{equation}
\small
F(p)=
\frac{\sum\limits_i \alpha_i\,
\exp\!\big(-\tfrac{1}{2}(p-\tilde{\mu}_i)^{\!\top}\tilde{\Sigma}_i^{-1}(p-\tilde{\mu}_i)\big)\, \tilde{f}_i}
{\sum\limits_j \alpha_j\,
\exp\!\big(-\tfrac{1}{2}(p-\tilde{\mu}_j)^{\!\top}\tilde{\Sigma}_j^{-1}(p-\tilde{\mu}_j)\big)}.
\label{eq:gs2d}
\end{equation}
The range term in JBU can be \emph{absorbed} by lifting to the joint domain (Theorem~\ref{thm:jbu_as_gs_joint}), or, equivalently, by keeping the domain 2D and letting the amplitude be query-dependent,
\(
\alpha_i(p)
=\exp\!\big(-\tfrac{\|I(p)-I(\tilde{\mu}_i)\|^2}{2\sigma_r^2}\big).
\)
The former is strictly \emph{query-independent} and thus mathematically cleaner; the latter matches common GS implementations with view-dependent weights.

\begin{theorem}[Specialization of GSJBU to JBU (isotropic limit)]
\label{thm:iso_limit}
Consider the anisotropic per-center model:
\begin{equation}
\small
\begin{aligned}
F(p)
&=
\frac{
\sum\limits_{q} f_q\,
\exp\!\big(-\tfrac{1}{2}(p-q)^{\!\top}\Sigma_q^{-1}(p-q)\big)\,
\beta_q(p)
}{
\sum\limits_{q}
\exp\!\big(-\tfrac{1}{2}(p-q)^{\!\top}\Sigma_q^{-1}(p-q)\big)\,
\beta_q(p)
},
\\[-2pt]
\beta_q(p)
&:=
\exp\!\big(-\tfrac{\|I(p)-I(q)\|^2}{2\sigma_r^2(q)}\big).
\end{aligned}
\label{eq:gsjbu_general}
\end{equation}
If $\Sigma_q \to \sigma_s^2 I_2$ and $\sigma_r(q)\to\sigma_r$ for all $q$, then \eqref{eq:gsjbu_general} reduces exactly to JBU \eqref{eq:jbu}.
\end{theorem}

\begin{table*}[t!]
\centering
\resizebox{\linewidth}{!}{
\begin{tabular}{l|c|c|l|l}
\toprule
\textbf{Parameter} & \textbf{Symbol} & \textbf{Default} & \textbf{Role} & \textbf{Rationale} \\
\midrule
Spatial sigma (x) 
    & $\sigma_x$ 
    & $\text{init}=\text{scale}$ (e.g., 16) 
    & Controls major-axis smoothing; receptive-field size 
    & Initialized proportional to upsampling factor to provide a wide prior; refined by TTO. \\
Spatial sigma (y) 
    & $\sigma_y$ 
    & Same as $\sigma_x$ 
    & Controls minor-axis smoothing 
    & Same reasoning as $\sigma_x$; enables anisotropy to emerge during TTO. \\
Orientation 
    & $\theta$ 
    & $0$ 
    & Rotation of the anisotropic Gaussian 
    & Zero-init avoids directional bias; TTO discovers optimal orientation. \\
Range sigma 
    & $\sigma_r$ 
    & $0.12$ 
    & Sensitivity to appearance/color similarity 
    & Moderate color differences ($\Delta I \approx 0.2\sim0.3$) are significantly downweighted; acts as soft bilateral prior. \\
Support radius (max) 
    & $R_{\max}$ 
    & $4\text{--}8$ 
    & Upper bound on spatial Gaussian support 
    & Balances context capture and cost ($\mathcal{O}((2R_{\max}+1)^2)$); too small truncates optimal kernels. \\
Dynamic multiplier 
    & $\alpha_{\mathrm{dyn}}$ 
    & $2.0$ 
    & Converts $\sigma_{\mathrm{eff}}$ to effective support radius 
    & Ensures coverage of $\sim95\%$ Gaussian mass ($2\sigma$ rule); prevents under-coverage early in TTO. \\
Center mode 
    & -- 
    & \texttt{nearest} 
    & Determines LR anchor for each HR pixel 
    & Nearest-center alignment improves stability and avoids aliasing for large upsampling factors. \\
\bottomrule
\end{tabular}
}
\caption{\textbf{Hyperparameter settings for \emph{Upsample Anything}.}
All parameters act as soft priors; the effective kernel shape is governed by test-time optimization of pixelwise anisotropic Gaussians.}
\label{tab:hyperparams}
\end{table*}

\begin{proof}
\small
Substitute $\Sigma_q=\sigma_s^2 I_2$ and $\sigma_r(q)=\sigma_r$ into \eqref{eq:gsjbu_general} to recover the numerator/denominator of \eqref{eq:jbu}.
\end{proof}

\begin{proposition}[Discrete-to-continuous convergence]
\label{prop:riemann}
Assume $F_{\mathrm{lr}}$ admits a bandlimited (or Lipschitz-continuous) interpolation $\tilde{F}:\Omega\to\mathbb{R}^{C}$,
and let the LR grid spacing be $\Delta x$.
Then as $\Delta x\to0$,
\begin{equation}
\small
\sum\nolimits_{q} \tilde{F}(q)\,\mathcal{G}_{\Lambda}\!\big(\phi(p),\phi(q)\big)\,(\Delta x)^2
\ \longrightarrow\
\int_{\Omega} \tilde{F}(x)\,\mathcal{G}_{\Lambda}\!\big(\phi(p),\phi(x)\big)\,dx,
\end{equation}
and the corresponding normalized ratios converge as well. Hence, discrete JBU converges to its continuous lifted-domain GS counterpart.
\end{proposition}

\begin{proof}[Sketch]
\small
$\mathcal{G}_{\Lambda}(\phi(p),\phi(\cdot))$ is bounded and continuous for fixed $p$.
Under the stated regularity, Riemann sums converge to the integral and the denominators stay strictly positive (finite kernel mass). The ratio convergence follows by standard arguments (e.g., dominated convergence and continuity of division on $\mathbb{R}\setminus\{0\}$).
\end{proof}

\paragraph{Consequences.}
(i) \emph{Equivalence in the joint domain} (Thm.~\ref{thm:jbu_as_gs_joint}) shows that JBU is a GS evaluation on $(x,I(x))$ with block-diagonal covariance.
(ii) \emph{Anisotropic generalization} \eqref{eq:gsjbu_general} recovers JBU in the isotropic limit (Thm.~\ref{thm:iso_limit}), and enables per-center covariance learning (our GSJBU).
(iii) \emph{Discrete-to-continuous consistency} (Prop.~\ref{prop:riemann}) justifies replacing sums by integrals when refining the sampling grid.

\paragraph{Implementation note.}
In practice we adopt \eqref{eq:gsjbu_general} with test-time optimization of $(\Sigma_q,\sigma_r(q))$.
For stability, $\Sigma_q\succ0$ is parameterized via $R(\theta_q)\mathrm{diag}(\sigma_x^2(q),\sigma_y^2(q))R(\theta_q)^{\!\top}$ with $\sigma_x,\sigma_y{>}0$.

\begin{figure}[t!]
    \centering
    \includegraphics[width=1.0\columnwidth]{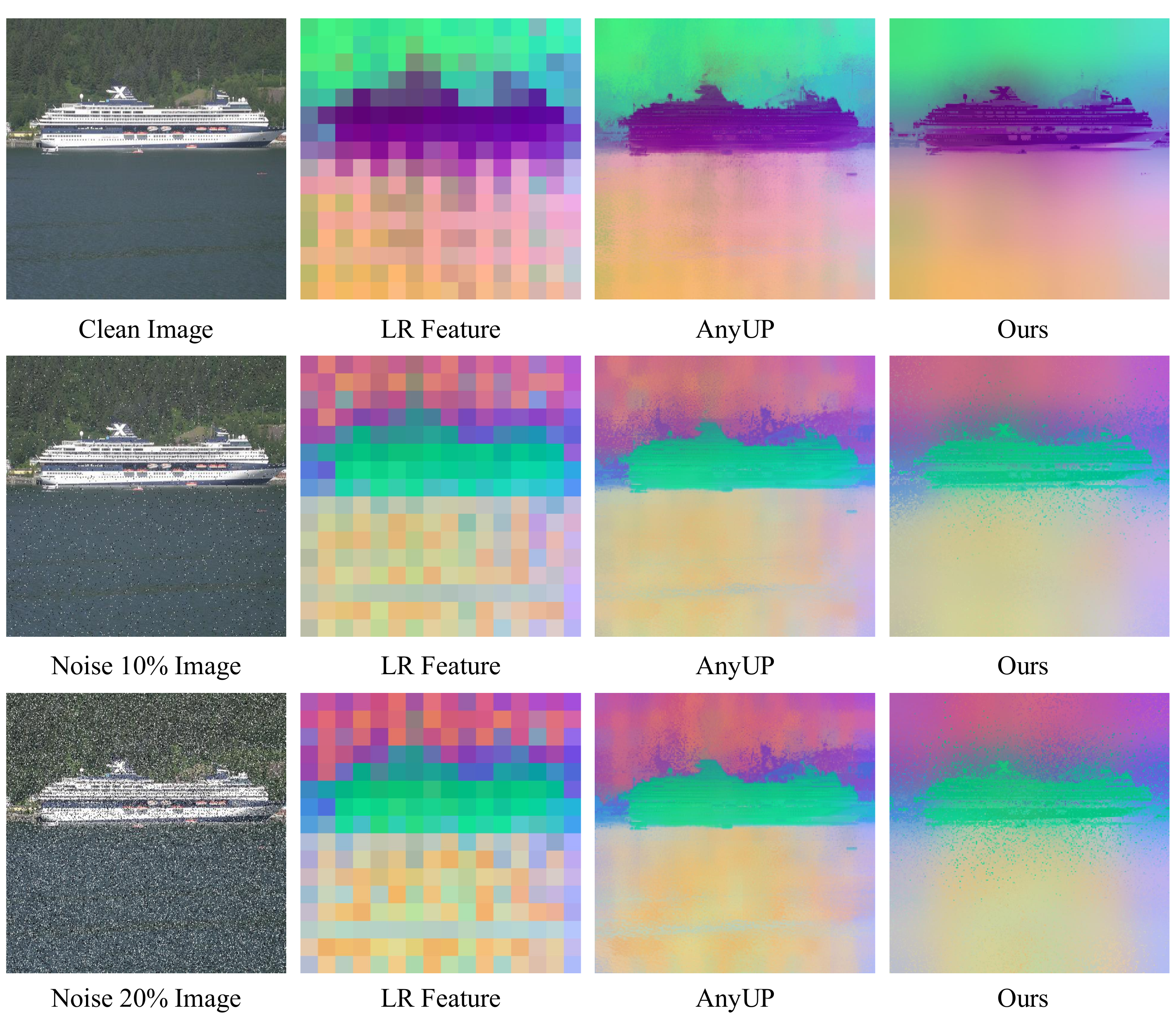}
    \caption{
    Qualitative comparison under low-SNR and noise corruption. From left to right: RGB input, low-resolution feature, upsampled feature by AnyUp, and ours (Upsample Anything). From top to bottom: clean image, 10\% noise, and 20\% noise. AnyUp remains stable under noise, while our TTO-based method overfits to noisy pixels, revealing its limitation when directly optimizing on corrupted inputs.
    }
\label{fig:vis_noise}
\end{figure}

\section{Hyperparameter Table}
The hyperparameters in Table~\ref{tab:hyperparams} function primarily as soft priors for test-time optimization. 
Since all spatial and range parameters are refined during the 50 optimization steps, the final performance depends only weakly on their initial values. 
A well-chosen initialization simply accelerates convergence, whereas  suboptimal values are eventually corrected by the optimization itself.
The table therefore summarizes practical initialization rules rather than strict hyperparameter requirements.
These rules are based on the expected receptive-field size, the dynamic range of the guidance image, and the desired locality prior, and they  lead to stable and fast convergence.

\section{Limitation under Low-SNR or Corrupted Inputs}

Although our method performs robustly across diverse datasets and even under moderate perturbations such as those in ImageNet-C, it exhibits a clear limitation when applied to images with extremely low signal-to-noise ratios or severe corruption.

Because our framework performs test-time optimization (TTO) by reconstructing the input image itself, the optimization process inherently assumes that the image contains a clean and reliable signal. When the input is degraded by noise—such as salt-and-pepper artifacts or heavy sensor perturbations—the model tends to overfit to these corruptions rather than recovering the underlying structure. Figure~\ref{fig:vis_noise} illustrates this effect: the first row shows results on a clean image, while the second and third rows demonstrate increasing corruption levels of 10\% and 20\%, respectively.

Despite the noise, pretrained Vision Foundation Models still produce reasonable feature embeddings, and AnyUp remains stable by directly upsampling feature maps. In contrast, our TTO-based Upsample Anything reconstructs the noisy signal faithfully, which unintentionally amplifies noise in both the reconstructed RGB and upsampled feature domains.

This limitation is not unique to our method but is common across all TTO-based image restoration approaches that optimize directly on corrupted inputs.
While one could incorporate a denoising stage before optimization to alleviate this issue, we consider it outside the current scope.
In summary, AnyUp demonstrates higher robustness under corrupted or low-SNR conditions, whereas our Upsample Anything excels when inputs are visually clean or when handling multi-modal signals such as RGB-D or 3D features.

{
    \small
    \bibliographystyle{ieeenat_fullname}
    \bibliography{main}
}


\end{document}